\documentclass{article}

\usepackage{listings}
\usepackage{latexsym}
\usepackage{amssymb}
\usepackage{amsmath}
\usepackage{amsthm}
\usepackage{enumitem}
\usepackage{hyperref}
\usepackage[a4paper, total={5.3in, 8in}]{geometry}

\usepackage{microtype}
\usepackage{graphicx}
\usepackage{subfigure}
\usepackage{tikz}
\usepackage{pgfplots} 
\usepackage{verbatim}

\pgfplotsset{compat=newest}

\usepackage{stmaryrd}
\usepackage{makecell}
\usepackage{algorithmic}
\usepackage{algorithm}
\usepackage{authblk}

\usepackage{caption}

\makeatletter
\newenvironment{breakablealgorithm}
  {
   \begin{center}
     \refstepcounter{algorithm}
     \hrule height.8pt depth0pt \kern2pt
     \renewcommand{\caption}[2][\relax]{
       {\raggedright\textbf{\fname@algorithm~\thealgorithm} ##2\par}%
       \ifx\relax##1\relax 
         \addcontentsline{loa}{algorithm}{\protect\numberline{\thealgorithm}##2}%
       \else 
         \addcontentsline{loa}{algorithm}{\protect\numberline{\thealgorithm}##1}%
       \fi
       \kern2pt\hrule\kern2pt
     }
  }{
     \kern2pt\hrule\relax
   \end{center}
  }
\makeatother

\newcommand{\PL}{\mathrm{PL}}
\newcommand{\err}{\mathrm{err}}

\theoremstyle{plain}
\newtheorem{theorem}{Theorem}[section]

\newtheorem{corollary}[theorem]{Corollary}
\theoremstyle{definition}

\theoremstyle{remark}

\begin{document}

\title{Interpretable classifiers for tabular data via 
feature selection and discretization}




\author{Reijo Jaakkola}
\author{Tomi Janhunen}
\author{Antti Kuusisto}
\author{Masood~Feyzbakhsh~Rankooh}
\author{Miikka Vilander\footnote{The authors are listed in alphabetical order. This is a preprint of a paper in 
DAO-XAI 2024 (Data meets Applied Ontologies in Explainable AI) 
in Santiago de Compostela.
T. Janhunen, A. Kuusisto, M. F. Rankooh and M. Vilander were
supported by the Academy of Finland consortium project Explaining AI via Logic (XAILOG), grant numbers 345633 (Janhunen) and 345612 (Kuusisto). A. Kuusisto and M. Vilander were also supported by the Academy of Finland project Theory of computational logics, grant numbers 352419, 352420, 353027, 324435, 328987.}}
\affil{Tampere University, Finland}
\date{}









\maketitle

\begin{abstract}
We introduce a method for computing immediately human interpretable yet accurate classifiers from tabular data. The classifiers obtained are short Boolean formulas, computed via first discretizing the original data and then using feature selection coupled with a very fast algorithm for producing the best possible Boolean classifier for the setting. We demonstrate the approach via 12 experiments, obtaining results with accuracies comparable to ones obtained via random forests, XGBoost, and existing results for the same datasets in the literature. In most cases, the accuracy of our method is in fact similar to that of the reference methods, even though the main objective of our study is the immediate interpretability of our classifiers. We also prove a new result on the probability that the classifier we obtain from real-life data corresponds to the ideally best classifier with respect to the background distribution the data comes from.
\end{abstract}


\section{Introduction}\label{sec:introduction}

Explainability and human interpretability are becoming an
increasingly important part of research on machine learning. 
In addition to the immediate benefits of explanations and interpretability in 
scientific contexts, the capacity to provide explanations behind
automated decisions has already been widely addressed also on 
the level of legislation. 
For example, the European General Data Protection Regulation 
\cite{euregulationstuff} and California Consumer Privacy Act 
\cite{california-oag} both 
refer to the right of individuals to get explanations of
automated decisions concerning them.

This article investigates interpretability in the framework of 
tabular data. Data in tabular form is important for numerous scientific and real-life contexts, being often even regarded as 
the \emph{most important} data form: see, e.g., 
\cite{swartzarmon22,tabularsurvey}. 
%
%
%
%
%
%
%
%
%
%
%
While explainable AI (or XAI) methods custom-made for  pictures and text cannot be readily used as such in the setting of tabular data \cite{pawelczyk20}, numerous successful XAI methods for the tabular setting exist; see the survey \cite{realtabularsurvey} for an overview.

In the current paper, we focus on producing \emph{highly interpretable} classifiers for tabular data. By ``interpretable'' we here mean that the inner workings of the classifier are directly readable from the classifier itself, which contrasts with explaining the operation of an external, black-box classifier \cite{rudin2019explaining}.\footnote{However, our method could of course also be used to explicate existing black-box models in a model agnostic way via first generating data with the black-box model to be explained and then using the method to produce a closely corresponding interpretable classifier.} 
For most of our experiments, the classifiers we
obtain are \emph{globally} interpretable, meaning that the classifier itself immediately reveals how it works on every possible input. Global interpretability stems from the classifiers being extremely short logical formulas and thereby directly intelligible. Now, globally interpretable classifiers obtained by our method are automatically also \emph{locally} interpretable, meaning that they provide an easy way of explicating why any particular input was classified in a particular way. 
For a reasonably small part of our experiments, the classifiers obtained are only locally but not quite globally interpretable.



The classifiers we produce are given in the form of short Boolean DNF-formulas which can naturally be conceived as simply Boolean concepts. One of the key issues that makes our approach possible is the surprising power of successful feature selection. Indeed, in our formulas, we use very small 
numbers of attributes, making them short and thus interpretable. Already in \cite{holte}, it was observed that by using a single attribute, one could get, for many of
the 16 datasets studied in that paper, an accuracy not drastically different from the then state-of-the-art decision trees. We apply this ``simplicity first'' approach in our method, together with a very rough discretization of numerical attributes to Boolean form. While we focus on interpretability, we obtain also relatively \emph{accurate} classifiers. Indeed, our classifiers perform surprisingly well in comparison to widely recognized methods despite the fact that generally, there exists an obvious trade-off between
accuracy and interpretability. 

\subsection{Overview of our method and contributions}


Our method works as follows. We have a tabular dataset $S$ with numerical and categorical attributes $X_1, \dots , X_k$ and a binary target attribute $q$. Our goal is to produce a classifier for $q$ based on $X_1,\dots , X_k$. 
%
%
Our method is based on the following steps; \emph{see
Section \ref{sec:overfitting} for 
more details.}

\begin{enumerate}
\item 
We first discretize the data to Boolean form using a 
very rough method of chopping numerical 
predicates at the median. 
\item 
Then, for increasing numbers $\ell$, we use feature selection to choose $\ell$ suitable attributes. After that, we compute the \emph{best possible} Boolean
formula for predicting $q$. By ``best possible'' we
mean the Boolean formula has the least percentage of misclassified points among all Boolean
formulas using the $\ell$ chosen attributes.
In other words, the formula has the least empirical error with respect to the 
0-1 loss function. The formula is given in DNF-form. 
\item 
We use nested cross-validation to test the accuracy of the formulas and find the most accurate formula using up to 10 attributes (the parameter 10 can be adjusted, but the choice 10 turned out sufficient for our experiments: see Section \ref{sec:overfitting} for a discussion). 
We look through the sequence of formulas obtained above for different numbers $\ell$ of features. Among the formulas with accuracy within one percentage point of the most accurate formula in the sequence, we select the one with the least number of features. We use the attributes in this formula to train the final formula via the entire training data.
Finally, we slightly 
simplify the final obtained formula---keeping it in DNF-form---via using a 
standard method from the literature.
\end{enumerate}

To demonstrate the robustness of our method, the discretization step is performed very roughly. Also the feature selection procedures are not critical to our approach; we use three readily available ones and choose the best final formula.

Concerning step 2, for computing the best possible DNF-formula, we present a very efficient algorithm running in time $\mathcal{O}(|W||\tau||T|)$, where $|W|$ is the number of rows in the data, $|\tau|$ the number of \emph{selected}
features and $|T| \leq \min(|W|,2^{|\tau|})$ the number of different \emph{row types} realized in the data with the selected features. The algorithm is fixed-parameter linear when $|\tau|$ is bounded by a constant---which it is in our scenario. 
If $|\tau|$ was not constant, the algorithm would have
running time $\mathcal{O}(|W|^2)$. 
Now, all our tests ran fast and smoothly on a laptop even with large datasets containing up to 423 680 rows. For six of the datasets, the runs took less than 30 seconds per split. Of the remaining datasets, five took less than 12 minutes per split. Even the two largest datasets had a runtime of less than an hour. The hyperparameter optimizations of random forests and XGBoost took systematically longer than running our method. For example, for the largest dataset, the hyperparameter optimization for XGBoost took roughly two and half hours per split.

As already indicated, we test our method on 12 tabular datasets: seven binary classification tasks from the UCI machine learning repository; a high-dimensional binary classification task with biological data; and four benchmark sets from \cite{grinsztajn2022why}. See Section \ref{sec:experiments} for a detailed list. 
In each test, we compare the accuracy of our classifier to a result obtained by state-of-the art classifiers. Now, taking into account that tabular data is still a challenge to deep learning methods \cite{grinsztajn2022why,tabularsurvey}, we use both
XGBoost \cite{xgboost} and Sklearn's implementation of random forests \cite{breiman2001random} as reference methods. These models are widely recognized for their
performance on tabular data \cite{swartzarmon22,grinsztajn2022why}. For the UCI datasets we additionally compare to the method of \cite{jelia23} based on the length of formulas. This method is computationally too inefficient to use on the other datasets we consider.
In addition to the reference methods, in
relation to 10 of the experiments, 
we also report results found in the literature for the used datasets. 

For 11 datasets we use standard ten-fold cross-validation and report the average accuracy as well as the standard deviation over the ten splits. For the high-dimensional dataset Colon with less than 100 data points we instead use leave-one-out cross-validation.

All formulas encountered in our experiments are small enough to be easily \emph{locally} interpretable. In the local interpretation process, the interpreter has a row of classified data which is then compared to the conjunctions of the DNF-classifier our method produces. A positively classified row will match with precisely one conjunction of the DNF-classifier, and a negatively interpreted one will clash with at least one attribute of each conjunction.

In addition to local interpretability, most of the formulas produced in the experiments are in fact so simple that they can be readily \emph{globally} interpreted, meaning that their behaviour on any input is clear simply by the form of the formula. Globally interpreting a short DNF-formula involves looking at each of the few conjunctions separately and interpreting the meaning of their particular combination of attributes. The behaviour of the entire formula is then given by the fact that it accepts the cases given by the conjunctions and rejects everything else.

As a concrete example of an interpretable classifier, we get the formula 
\[
\neg p_1 \lor (\neg p_2 \land \neg p_3)
\] 
from an experiment on a Breast Cancer dataset concerning the benignity of breast tumors. The attributes $p_1$, $p_2$ and $p_3$ relate to measures of uniformity of cell size, bare nuclei and bland chromatin. 
The accuracy of this formula on the corresponding test data is 94.1 percent. The average accuracy of our method over ten splits of the Breast Cancer data is 95.9 percent, while XGBoost obtains 97.1 percent and  random forests 97.4 percent. We stress that our formula is, indeed, highly interpretable, while the classifiers obtained by XGBoost and random forests are of very large and of a black box nature.

As another example, from a Colon dataset we get the formula 
\[
p_1 \lor p_2
\]
from the majority of our tests. 
The attributes $p_1$ and $p_2$ have been selected from a total of 5997 Booleanized attributes related to genes. The accuracy of our method on this dataset is 80.6 percent, while random forests get 83.9 percent and XGBoost 72.6 percent. 

\textbf{In relation to accuracy, the experiments are summarized in Figure \ref{fig:resultsum}.} In Table \ref{tab:featurenumbers} we report the average number of attributes (features) used by the final formulas and the dimensions of the datasets.

In general, our method has similar accuracy to the reference methods. 
For the UCI datasets reported on the left in Figure \ref{fig:resultsum}, our method obtains accuracies on par with and sometimes better than the state-of-the-art reference methods. For example, on the Hepatitis dataset we obtain an average accuracy of 80.7 percent compared to the 78.2 percent of random forests, 79 percent of XGBoost and 79.4 percent of the formula-size method of \cite{jelia23}. The four benchmark datasets from \cite{grinsztajn2022why} reported on the right in Figure \ref{fig:resultsum} were the most difficult for our method, but even for these we obtained some surprisingly accurate and reasonably interpretable formulas. 
For example we obtain a formula of the form
\begin{align*}
(p_1 &\land p_2 \land p_3 \land \neg p_4) \lor (p_2 \land \neg p_1 \land \neg p_5 \land \neg p_4) \\
&\lor (p_4 \land p_3 \land \neg p_1 \land \neg p_5 \land \neg p_2)
\end{align*}
from an experiment on the RoadSafety dataset. This formula is quite short and has a test accuracy of 73.2 percent compared to the average 83.2 percent achieved by black box classifiers.


Concerning our experiments, we stress once more that the main advantage of our method is interpretability. As Table \ref{tab:featurenumbers} indicates, most of our classifiers use very small numbers of features, thus being interpretable. Some classifiers obtained in the tests are longer, but still within the bounds of local (while not necessarily global) interpretability.

Finally, a key insight behind our method is the idea of computing the
ideal classifiers (i.e., the above mentioned best possible 
Boolean classifiers) and the fact that this can be 
done fast using the $\mathcal{O}(|W||\tau||T|)$ algorithm (where
$|T| \leq \min(|W|,2^{|\tau|})$)
when $\tau$ is small. Since small $\tau$ is 
often sufficient---which is another key insight in
the method---the approach indeed works quite 
accurately and fast. The ideal classifiers being 
central to the approach, we call the method 
\emph{the ideal classifier method}. A possible 
alternative for this is would be the \emph{ideal 
DNF-method}.

In addition to experiments, we also prove a novel theoretical sample bound result that
can be used for estimating whether a Boolean classifier obtained 
from data is in fact an ideal Boolean classifier with respect to
the background distribution the data comes from. See Section \ref{samplesizesection} for the related theorem and the Appendix for the proof. Our result is in flavour similar to the various results in statistics that estimate the sizes of samples needed for obtaining a given confidence interval; see, e.g., \cite{StatisticsBook} for further details. Results of this form can indeed be useful for estimating if the
classifiers we 
obtain from datasets are in fact best possible classifiers 
(for the given features) with respect to the underlying probability distribution the data originates from. In Section \ref{samplesizesection} we illustrate how our result can be used in practice in relation to the datasets of the current paper. 

%
%
%
%
%
%
%
%
%
%
%
%

\subsection{Further 
related work}

Concerning further related work, while the literature on explainability is rather extensive, only a relatively small
part of it is primarily based on logic. 
See \cite{silva} for a survey on logic-based explainability, called \emph{formal 
explainable} AI (or FXAI) there. We mention here the 
two prominent works dealing with minimality notions for 
Boolean classifiers and pioneering much of the recent work in FXAI, \cite{ShihCD18}
and \cite{alexey}. 

Like the articles \cite{ShihCD18}
and \cite{alexey}, in fact most of logic-based explainability differs significantly from the current paper. Firstly, most papers in the field concern local rather than global explanations, and also inherent interpretability (as opposed to explainability of existing classifiers) is rarely the main focus. However, there is one method that is close enough to ours to require an explicit and direct analysis in the current paper. That method---let us here call it the \emph{formula-size method} (FSM)---is investigated in \cite{jelia23}. Just like the current work, FSM uses a validation-based approach to avoid overfitting and find Boolean formulas with a small error over real-life datasets. A major difference between our approach and FSM is that in our method, we investigate increasing \emph{numbers of features} as opposed to increasing \emph{length bounds} on formulas as in FSM. The algorithm of FSM searches through the space of all possible formulas of increasing lengths, making it impossible to use with larger datasets. This contrasts with the fixed-parameter linear $\mathcal{O}(|W|2^{|\tau|})$ algorithm we use with $|\tau|$ being a constant.\footnote{In our experiments we constrain $|\tau|$ to be at most $10$.} Our algorithm outputs classifiers with the minimum error in relation to the set of input features
used, whereas FSM optimizes
with respect to formula size. Together with other experiments, we also describe in Section \ref{sec:results} tests we ran to
compare FSM with our method.

Concerning yet further related work on interpretable AI, the articles \cite{OptimalRuleLists,OptimalScoringSystems} investigate the use of sparse rule lists and scoring systems which are optimized with respect to their error and size. These models are sparse in the sense that they try to use a small number of features, which makes them interpretable. The empirical results reported in these papers also demonstrate the surprising effectiveness of these interpretable models on real-world tabular data. Using the methods proposed in the articles \cite{OptimalRuleLists,OptimalScoringSystems} requires---as in the case of FSM in \cite{jelia23}---solving very resource-consuming combinatorial optimization problems, since in addition to their error, classifiers are also optimized with respect to their size. For yet further related work on interpretable AI, see \cite{rudin2019explaining}.




\section{Preliminaries}\label{sec:preliminaries}




A \textbf{vocabulary} is a finite set of 
symbols $p_i$ referred to as 
\textbf{proposition symbols}. 
For a  vocabulary $\sigma=\{p_1,\dots,p_k\}$ the syntax of propositional logic $\PL[\sigma]$ over $\sigma$ is given by the grammar
$\varphi ::= p \mid \neg \varphi \mid \varphi \land \varphi \mid \varphi \lor \varphi$,
where $p \in \sigma$. We also define the exclusive or $\varphi \oplus \psi := (\varphi \lor \psi) \land \neg(\varphi \land \psi)$ as an abbreviation. A formula $\varphi \in \PL[\sigma]$ is in \textbf{disjunctive normal form} (DNF) if $
\varphi = \bigvee_{i = 1}^m \psi_i$,
where each $\psi_i$ is a conjunction of \textbf{literals} (i.e., formulas $p$ or $\neg p$ where $p \in \sigma$).

A $\sigma$-model is a structure $M = (W,V)$, where $W$ is a finite non-empty set called the \textbf{domain} of $M$ and $V: \sigma \to \mathcal{P}(W)$ is a \textbf{valuation} which assigns to each $p \in \sigma$ the set of points $V(p) \subseteq W$ where $p$ is true. Such a valuation extends in the usual way into a valuation $V : \PL[\sigma] \to \mathcal{P}(W)$, with $\wedge$, $\vee$ and $\neg$ corresponding to the intersection, union and complementation (with respect to $W$) operations. A formula $\varphi \in \PL[\sigma]$ is true in the point $w \in W$ of a $\sigma$-model $M = (W, V)$, denoted $M, w \models \varphi$, if $w \in V(\varphi)$. Note that thereby each formula $\varphi$ corresponds to a subset $V(\varphi)$ of $W$. For $\varphi, \psi \in \PL[\tau]$, we define $\varphi \vDash \psi$ if for all models $M$ and all $w \in W$, $M, w \vDash \varphi$ implies $M, w \vDash \psi$.

A \textbf{$\sigma$-type} $t$ is a conjunction such that for each $p \in \sigma$, precisely one of the literals $p$ and $\neg p$ is a conjunct of $t$ and $t$ has no other conjuncts. We assume some standard bracketing and ordering of literals so that there are exactly $2^{|\sigma|}$ $\sigma$-types. We denote the set of $\sigma$-types by $T_{\sigma}$. Note that each point $w \in W$ of a $\sigma$-model $M = (W, V)$ satisfies exactly one $\sigma$-type. Thus, $\sigma$-types induce a partition of the domain $W$. On the other hand, from the truth table of a formula $\varphi \in \PL[\sigma]$ one can obtain an equivalent formula $\psi$ that is a disjunction of types and thus in DNF.


For a vocabulary $\tau$, let $\mu : T_{\tau \cup \{q\}} \to [0,1]$ be a probability distribution. Here $q$ is the separate target attribute of the classification task. For $t \in T_\tau$, we define $\mu(t) = \mu(t \land q) + \mu(t \land \neg q)$. The distribution $\mu$ corresponds to the real-world phenomenon that gives rise to practical data. Thus, we generally assume that $\mu$ is unknown and define the \textbf{true error} (or \textbf{risk}) of a formula $\varphi \in \PL[\tau]$ with respect to $\mu$ as
\[\err_\mu(\varphi) := \Pr_{t \sim \mu}[t \models \varphi \oplus q] = \sum\limits_{\substack{t\, \in\, T_{\tau \cup \{q\}} \\ \, t \vDash \varphi \oplus q}} \mu(t) .\]
This is the probability that $\varphi$ disagrees with $q$. Our goal is to obtain formulas with a small true error. This is made difficult by the fact that $\mu$ is unknown. We can, however, estimate the true error via available data.

Let $M = (W, V)$ be a $\tau \cup \{q\}$-model. For us, $M$ corresponds to the available tabular data. Given a propositional formula $\varphi \in \PL[\tau]$, we define the \textbf{empirical error} (or \textbf{empirical risk}) of $\varphi$ with respect to $M$ as
\[\err_M(\varphi) := \frac{|V(\varphi \oplus q)|}{|W|}.\]
The empirical error $\err_M(\varphi)$ is easily computable as the proportion of points where $\varphi$ disagrees with $q$. If $M$ is fairly sampled from $\mu$, then by the law of large numbers $\err_M(\varphi) \to \err_\mu(\varphi)$ almost surely when $|W| \to \infty$.

Given a distribution $\mu$, the formula
\[
    \varphi^\mu_{\mathrm{id}} := \bigvee \bigg\{t \in T_\tau \mid \frac{\mu(t \land q)}{\mu(t)} \geq 1/2\bigg\},
\]
which we call the \textbf{ideal classifier}, has the smallest true error with respect to $\mu$ among the formulas in $\PL[\tau]$. This is a syntactic, logic-based representation of what is known as the \textbf{Bayes classifier} in the literature \cite{lugosi},  not to be confused with naive Bayesian classifiers. A Bayes classifier always gives the best possible prediction, and clearly so does an ideal classifier. 
Now, as $\mu$ is unknown, the ideal classifier is again a theoretical goal for us to approximate.

Given a $\tau \cup \{q\}$-model $M = (W, V)$, the formula
\[\varphi^M_{\mathrm{id}} := \bigvee \bigg\{t \in T_\tau \mid \frac{|V(t \land q)|}{|V(t)|} \geq 1/2\bigg\},\]
which we call the \textbf{empirical ideal classifier}, has the smallest empirical error with respect to $M$ among the formulas in $\PL[\tau]$. This formula is easily computable and an essential tool of our study.

\section{Feature selection and overfitting}\label{sec:overfitting}

Our general goal is to use, for a suitable set
$\tau$ of attributes, the empirical ideal classifier to approximate the ideal classifier. In this section, we specify our methodology and show bounds on sufficient sample size to guarantee that the two classifiers are identical with high probability.

Let $\tau$ be a small set of promising attributes chosen from the initially possibly large set of all attributes.
We describe a quadratic time algorithm to obtain the empirical ideal classifier. 
The pseudocode \textbf{Algorithm \ref{algo:overfitter}} below describes a formal implementation of this algorithm.
Basically, we scan the points $w \in W$ once. 
%
For every $\tau$-type $t$ that is realized in $M = (W,V)$, we initiate and maintain two counters, $n_t$ and $c_t$. The first counter $n_t$ counts how many times $t$ is realized in $M$, while $c_t$ counts how many times $t \land q$ is realized in $M$. The number $c_t/n_t$ is then the probability $|V(t \land q)|/|V(t)|$. The empirical ideal classifier $\varphi_{id}^{M}$ can be constructed by taking a disjunction over all the types $t$ which are realized in $M$ and for which $c_t/n_t \geq 1/2$.

\begin{breakablealgorithm}
\caption{Compute the ideal classifier $\varphi^M_{id}$}\label{algo:overfitter}
\textbf{Input:} a $(\tau \cup \{q\})$-model $M = (W,V)$
\begin{algorithmic}[1]
\STATE $T_M \leftarrow \varnothing$ \COMMENT{All the $\tau$-types realized in $M$ will be stored in the set $T_M$}
\STATE \textbf{for} $w \in W$ \textbf{do}
\STATE \quad \quad $t \leftarrow$ the $\tau$-type of $w$
\STATE \quad \quad \textbf{if} $t \not\in T_M$ \textbf{then}
\STATE \quad \quad \quad \quad $T_M \leftarrow T_M \cup \{t\}$
\STATE \quad \quad \quad \quad $n_t, c_t \leftarrow 0, 0$
\STATE \quad \quad $n_t \leftarrow n_t + 1$
\STATE \quad \quad \textbf{if} $w \in V(q)$ \textbf{then}
\STATE \quad \quad \quad \quad $c_t \leftarrow c_t + 1$
\STATE $\varphi_{id}^M \leftarrow \bot$
\STATE \textbf{for} $t \in T_M$ \textbf{do}
\STATE \quad \quad \textbf{if} $c_t / n_t \geq 1/2$ \textbf{then}
\STATE \quad \quad \quad \quad $\varphi_{id}^M \leftarrow \varphi_{id}^M \lor t$
\STATE \textbf{return $\varphi_{id}^M$}
\end{algorithmic}
\end{breakablealgorithm}

It is clear that this algorithm runs in polynomial time with respect to the size of $M$, the size being $\mathcal{O}(|W||\tau|)$. A more precise analysis shows that the running time of this algorithm is $\mathcal{O}(|W||T_M||\tau|),$ where $|T_M|$ counts the number of $\tau$-types that are realized in $M$. Since $|T_M| \leq |W|$, this gives a worst case time complexity of $\mathcal{O}(|W|^2|\tau|)$.  
If $\tau$ has a fixed size bound, as in our experiments below, then this reduces to a linear time algorithm. Moreover, we note that clearly the size $|\varphi_{id}^M|$ is  $\mathcal{O}(|W||\tau|)$. 

%
%

A full step-by-step description of our method follows:
\begin{enumerate} 
\item 
We begin with a tabular dataset with features $X_1, \dots , X_{m},q$, where $q$ is a Boolean target attribute. We denote this full training data by $W_0$.
%
%
%
\item We randomly separate 30 percent of $W_0$ as validation data $W_{\mathrm{val}}$. The remaining part we call the training data $W_{\mathrm{train}}$.
We then discretize (or Booleanize) 
both of these datasets, ending up with 
tabular datasets with 
strictly Boolean attributes $p_1,\dots , p_k,q$. To this end, we simply chop the numerical attributes at the median value of the training data $W_{\mathrm{train}}$, above median meaning ``yes'' and at most median corresponding to ``no''. For categorical attributes we use one-hot encoding. 
\item We iterate steps 4 and 5 for increasing numbers $\ell$ of features from 1 to 10.
\item
We run three feature selection procedures on the training data, each selecting $\ell$ of the attributes $p_1,\dots , p_k$ to be 
used for classification. 
\item 
Suppose one procedure selected the set $F = \{ p_{i_1},\dots , p_{i_{\ell}}\}$ of features.
We use Algorithm 1 to compute the empirical ideal classifier for $q$ on the data $W_{\mathrm{train}}$ using the attributes in $F$. Note that this is the \emph{best possible} classifier for $q$ over $F$, given in DNF-form. We do this step for all three sets of features given by the procedures. We select the formula with the highest validation accuracy $r$ on the data $W_{\mathrm{val}}$ and record the tuple $(\ell, F, r)$ for future steps.
%
%
%
%
%
%
\item 
Once step 5 has halted, we look back at the tuples $(\ell, F, r)$ recorded. We select the feature set $F_\ell$ corresponding to the smallest number $\ell$ with validation accuracy $r_\ell$ within one percentage point of the best accuracy in the full sequence
\[
(1, F_1, r_1), \dots, (10, F_{10}, r_{10}).
\] 
We run the Booleanization again using the median of the full training data $W_0$ and compute the best possible classifier one last time using the selected set $F_\ell$ of features and the data $W_0$. 
\item 
Finally, we simplify the selected DNF-formula via \texttt{simplify\_logic} from SymPy. This gives a potentially simpler logically equivalent DNF-formula. 
\end{enumerate} 

To demonstrate the robustness of our method, the discretization steps are performed very roughly, using
simply the medians. 
Also the specific feature selection procedures are not in any way custom-made for our method; we use three readily available ones (see Sect. \ref{sec:experiments} for details) and use the one with the best accuracy.



Regarding step 5, choosing the formula with the best validation accuracy is done to avoid overfitting to the training data. With the increase of the number of features, the validation accuracy generally first improves and at some point starts to decline. This decline is a sign of overfitting. As overfitting often occurs already for small numbers of selected attributes, we may regard overfitting as \emph{useful} for the method, leading to \emph{short} formulas. However, in some cases, the accuracy stagnates and the overfitting point seems hard to find. In these cases, we nevertheless stop at ten features, the maximum considered in the method. For some datasets it could be necessary to go further before the accuracy stagnates, but for our experiments, ten features was (more than) enough.


While the last formula obtained in step 5 can be quite long, step~6 helps to obtain shorter formulas. 
By choosing an earlier formula in the obtained sequence with almost the same accuracy, we can reduce the number of features used, drastically improving the interpretability of our formulas without sacrificing much accuracy. 

The formula obtained from step 6 is often very short and in most cases even 
globally interpretable. 
Nevertheless, as a final step, a standard formula simplifying tool  \texttt{simplify\_logic} from SymPy is used, and this gives a potentially even shorter DNF-formula.
We note that even 
without using \texttt{simplify\_logic}, all formulas encountered in our experiments are readily \emph{locally} interpretable. Recall from the Introduction that in the local interpretation process, a positively classified input will match with precisely one conjunction of the DNF-classifier, and a negatively interpreted one will clash with at least one attribute of each conjunction. 




\subsection{Bounds on sample size}\label{samplesizesection}

As our method consists of using the empirical ideal classifier as an approximation of the ideal classifier, we would like to have some guarantees on when the two classifiers are the same. We next present a theorem which tells us how large samples we need in order to be confident that the empirical ideal classifier is also the true ideal classifier.


The lower bound provided by our theorem depends in a crucial way on how ``difficult'' the underlying distribution $\mu$ is. More formally, we say that a probability distribution $\mu : T_{\tau \cup \{q\}} \to [0,1]$ \(\varepsilon\)-\textbf{separates} \(t \in T_\tau\), if we have that $|\mu(t \land q) - \mu(t \land \neg q)| \geq \varepsilon$. The larger the $\varepsilon$ is, the easier it is to detect via sampling which of the types $t \land q$ and $t \land \neg q$ has the higher probability of occurring.

The formal statement of the theorem is now as follows. See Appendix \ref{appendix:proof_of_the_theorem} for the proof.

\begin{theorem}\label{thm:sample_size}
    Fix a vocabulary $\tau$, a proposition symbol $q\not\in \tau$ and a probability distribution $\mu : T_{\tau \cup \{q\}} \to [0,1]$. Let $\varepsilon, \delta > 0$ and 
    \[n \geq \frac{2\ln(2^{|\tau| + 1}/\delta)}{\varepsilon^2}.\] 
    Then with probability at least $1 - \delta$, the empirical ideal classifier with respect to a sample $M$ of size $n$ agrees with the ideal classifier with respect to $\mu$ on every $t \in T_\tau$ which is \(\varepsilon\)-separated by \(\mu\). In particular, if $\mu$ \(\varepsilon\)-separates every \(t \in T_\tau\), then the empirical ideal classifier is the ideal classifier with probability at least $1 - \delta$. 
\end{theorem}

\begin{corollary}\label{cor:sample_size}
    Fix a vocabulary $\tau$, a proposition symbol $q\not\in \tau$ and a probability distribution $\mu : T_{\tau \cup \{q\}} \to [0,1]$. Let $\varepsilon, \delta > 0$ and 
    \[n \geq \frac{2^{2|\tau|+1}\ln(2^{|\tau| + 1}/\delta)}{\varepsilon^2}.\]
    Then with probability at least $1 - \delta$, we have that 
    \[\err_\mu(\varphi_{id}^M) < \err_\mu(\varphi_{id}^\mu) + \varepsilon.\]
\end{corollary}
\begin{proof}
    Fix \(\varepsilon, \delta > 0\). Given \(\eta > 0\) we know that if \(\varphi_{id}^M\) agrees with \(\varphi_{id}^\mu\) on every \(t \in T_\tau\) which is \(\eta\)-separated by \(\mu\), then
    \begin{align*}
    \err_\mu(\varphi_{id}^\mu) - \mathrm{err}_\mu(\varphi_{id}^M) 
    &= \sum_{\substack{t \in T_\tau \\ \mu \text{ does not \(\eta\)-separate } t}} |\mu(t \land q) - \mu(t \land \neg q)| \\
    &< |T_\tau| \eta.
    \end{align*}
    Setting \(\eta := \varepsilon / |T_\tau|\) and applying Theorem \ref{thm:sample_size} gives us the desired result.
\end{proof}

To illustrate the use of this latter bound, suppose that $|\tau| = 3, \delta = 0.01$ and $\varepsilon = 0.05$. Corollary \ref{cor:sample_size} shows that we need a sample $M$ of size at least $18887$ to know that with probability at least $0.99$ the theoretical error of \(\varphi_{id}^M\) is less than that of \(\varphi_{id}^\mu\) plus \(0.05\). The datasets that we consider in Section \ref{sec:experiments} contain three datasets (Covertype, Electricity, RoadSafety) that have at least $18887$ data points. Thus we can be fairly confident that for these datasets any empirical ideal classifier that uses at most three features obtains an accuracy which is similar to that of the true ideal classifier on those features.





\section{Experiments}\label{sec:experiments}

\subsection{Experimental setup}

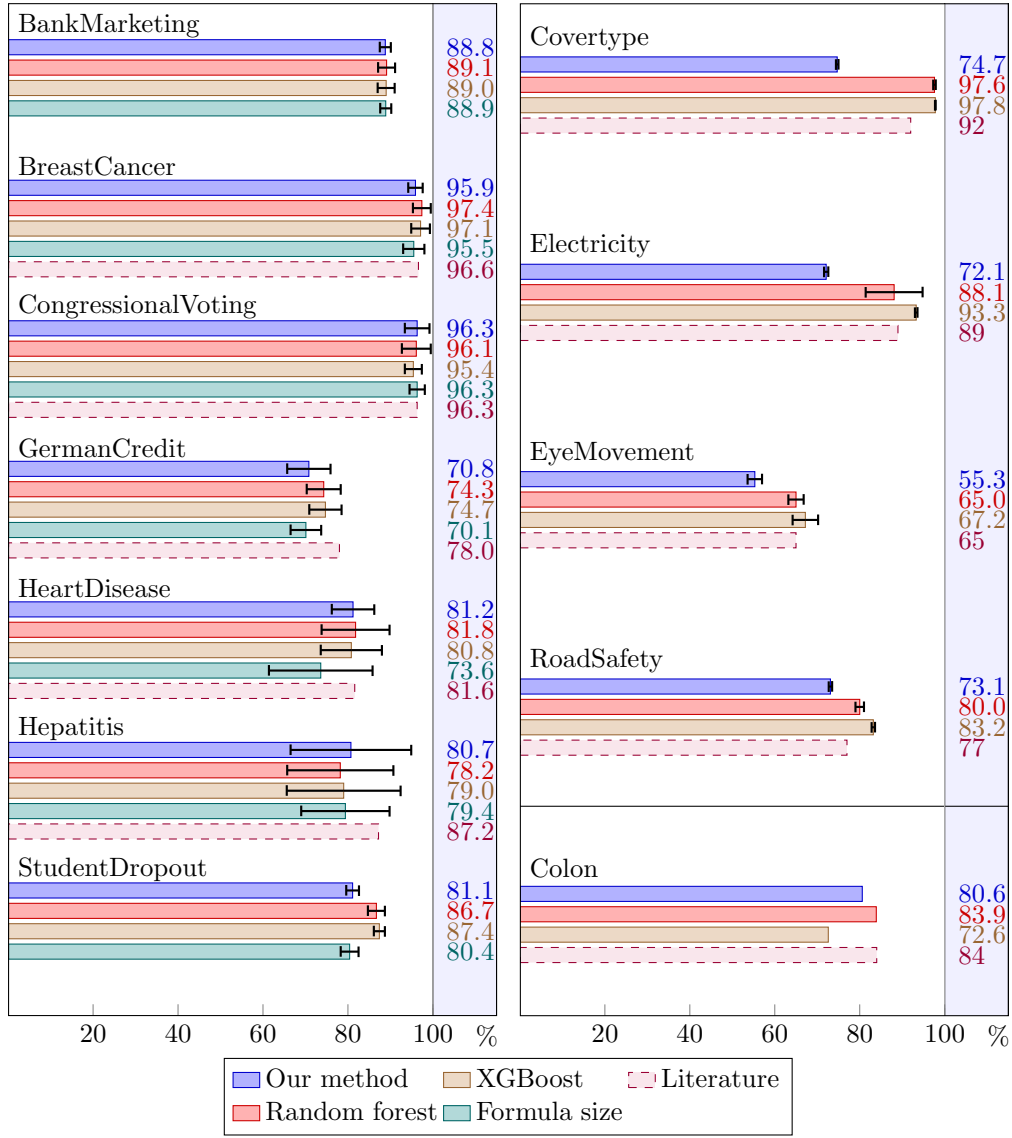
\begin{figure}[!tb]
\begin{minipage}{.5\textwidth}
\centering
\begin{tikzpicture}
    \begin{axis}[
            reverse legend,
            xbar,
            bar width=.2cm,
            width=8cm,
            height=15cm,
            legend style={at={(0,-0.05)},
                anchor=north west},
            legend columns = 2,
            legend image code/.code={
            \draw [#1] (-0.1cm,-0.1cm) rectangle (0.25cm,0.1cm); },
            axis on top,
            yticklabels={BankMarketing,BreastCancer,CongressionalVoting, GermanCredit, HeartDisease, Hepatitis, StudentDropout},
            ytick=data,
            xtick pos=left,
            ytick style={draw=none},
            ytick = {1,2,3,4,5,6,7},
            y dir=reverse,
            y tick label style={anchor=west,xshift=.15cm, yshift=5.5*\pgfkeysvalueof{/pgfplots/major tick length}},
            nodes near coords,
            visualization depends on={(101-\thisrow{accuracy}) \as \offset},
            node near coords style={shift={(axis direction cs:\offset,0)},/pgf/number format/.cd,fixed zerofill, precision=1},
            nodes near coords align={horizontal},
            xmin = 0.1, xmax = 115,
            xlabel={\%},
            x label style={at={(axis description cs:0.98,0)},anchor=north},
            minor xtick = {100},
            grid={minor}
        ]

        \draw[fill=blue!6!white] (100,0) rectangle (115, 8);

        \addplot[dashed, purple!80!black,fill=purple!10!white, 
                ] 
                table[x=accuracy, y=dataset, 
                col sep=&, row sep=crcr] {
                accuracy &   dataset             & deviation \\
                96.6       &   2        & 5 \\
                96.3       &   3 & 5 \\
                78.0       &   4        & 5 \\
                81.6       &   5        & 5 \\
                87.2       &   6           & 5 \\
        };

        \addplot[teal!80!black,fill=teal!30!white, error bars/.cd,
                error bar style = {thick, black},
                error mark options={
                rotate=90,
                mark size=2pt,
                line width=0.8pt
                },
                x dir=both,
                x explicit] table[x=accuracy, y=dataset, x error=deviation, col sep=&, row sep=crcr] {
                accuracy &  dataset &               deviation \\
                88.9    &   1   &       1.3 \\
                95.5    &   2    &       2.5 \\
                96.3    &   3 &   1.8 \\
                70.1      &   4    &     3.6 \\
                73.6    &   5    &       12.2 \\
                79.4    &   6       &       10.4 \\
                80.4    &   7  &       2.1 \\
        };

        \addplot[brown!80!black,fill=brown!30!white, error bars/.cd,
                error bar style = {thick, black},
                error mark options={
                rotate=90,
                mark size=2pt,
                line width=0.8pt
                },
                x dir=both,
                x explicit] table[x=accuracy, y=dataset, x error=deviation, col sep=&, row sep=crcr] {
                accuracy &  dataset &               deviation \\
                89.0    &   1   &       2.0 \\
                97.1    &   2    &       2.2 \\
                95.4    &   3 &   2.0 \\
                74.7      &   4    &     3.8 \\
                80.8    &   5    &       7.2 \\
                79.0    &   6       &       13.4 \\
                87.4    &   7  &       1.3 \\
        };

        \addplot[red!80!black,fill=red!30!white, error bars/.cd,
                error bar style = {thick, black},
                error mark options={
                rotate=90,
                mark size=2pt,
                line width=0.8pt
                },
                x dir=both,
                x explicit] table[x=accuracy, y=dataset, x error=deviation, col sep=&, row sep=crcr] {
                accuracy &  dataset &               deviation \\
                89.1    &   1   &       2.0 \\
                97.4    &   2    &       2.1 \\
                96.1    &   3 &   3.4 \\
                74.3    &   4    &       4.0 \\
                81.8   &   5    &       8.0 \\
                78.2    &   6       &       12.5 \\
                86.7    &   7  &       2.0 \\
        };
        
        \addplot[blue!80!black,fill=blue!30!white, error bars/.cd,
                error bar style = {thick, black},
                error mark options={
                rotate=90,
                mark size=2pt,
                line width=0.8pt
                },
                x dir=both,
                x explicit] table[x=accuracy, y=dataset, x error=deviation, col sep=&, row sep=crcr] {
                accuracy &   dataset             & deviation \\
                88.8       &   1       & 1.3   \\
                95.9       &   2        & 1.7 \\
                96.3       &   3 & 2.9 \\
                70.8       &   4        & 5.1 \\
                81.2       &   5        & 5 \\
                80.7       &   6           & 14.2 \\
                81.1       &   7      & 1.5 \\
        };

    \end{axis}
\end{tikzpicture}
\end{minipage}%
\begin{minipage}{.5\textwidth}
\centering
\begin{tikzpicture}
    \begin{axis}[
            reverse legend,
            xbar,
            bar width=.2cm,
            enlarge y limits=0.11,
            width=8cm,
            height=15cm,
            legend style={at={(0,-0.05)},
                anchor=north west},
            legend columns=2,
            legend image code/.code={
            \draw [#1] (-0.1cm,-0.1cm) rectangle (0.25cm,0.1cm); },
            axis on top,
            ytick=data,
            xtick pos=left,
            ytick style={draw=none},
            ytick = {
            2, 3, 4, 5, 6},
            yticklabels = {
            Covertype, Electricity, EyeMovement,  RoadSafety, Colon},
            y dir=reverse,
            y tick label style={anchor=west,xshift=.15cm, yshift=5*\pgfkeysvalueof{/pgfplots/major tick length}},
            nodes near coords,
            visualization depends on={(101-\thisrow{accuracy}) \as \offset},
            node near coords style={shift={(axis direction cs:\offset,0)},/pgf/number format/.cd,fixed zerofill, precision=1},
            nodes near coords align={horizontal},
            xmin = 0.1, xmax = 115,
            xlabel={\%},
            x label style={at={(axis description cs:0.98,0)},anchor=north},
            minor xtick = {100},
            grid={minor},
        ]
        
       \draw[fill=blue!6!white] (100,0) rectangle (115, 8);

       \draw (0, 5.43) -- (115, 5.43);
        

         \addplot[dashed, purple!80!black,fill=purple!10!white, 
                node near coords style={precision=0}
                ] 
                table[x=accuracy, y=dataset, 
                col sep=&, row sep=crcr] {
               accuracy &   dataset             & deviation \\
                84       &   6            &    \\
                92       &   2        & 5.0 \\
                89       &   3      & 5.0 \\
                65       &   4      & 5.0 \\
                77       &   5       & 5.0 \\
        };

        \addplot[brown!80!black,fill=brown!30!white, error bars/.cd,
                error bar style = {thick, black},
                error mark options={
                rotate=90,
                mark size=2pt,
                line width=0.8pt
                },
                x dir=both,
                x explicit] table[x=accuracy, y=dataset, x error=deviation, col sep=&, row sep=crcr] {
                accuracy &   dataset             & deviation \\
                72.6       &   6            &    \\
                97.8       &   2        & 0.08 \\
                93.3       &   3      & 0.3 \\
                67.2       &   4      & 3.0 \\
                83.2       &   5       & 0.4 \\
        };

        \addplot[red!80!black,fill=red!30!white, error bars/.cd,
                error bar style = {thick, black},
                error mark options={
                rotate=90,
                mark size=2pt,
                line width=0.8pt
                },
                x dir=both,
                x explicit] table[x=accuracy, y=dataset, x error=deviation, col sep=&, row sep=crcr] {
                accuracy &   dataset             & deviation \\
                83.9       &   6            &    \\
                97.6       &   2        & 0.3 \\
                88.1       &   3      & 6.7 \\
                65.0       &   4      & 1.8 \\
                80.0       &   5       & 1.0 \\
        };
        
        \addplot[blue!80!black,fill=blue!30!white, error bars/.cd,
                error bar style = {thick, black},
                error mark options={
                rotate=90,
                mark size=2pt,
                line width=0.8pt
                },
                x dir=both,
                x explicit] table[x=accuracy, y=dataset, x error=deviation, col sep=&, row sep=crcr] {
                accuracy &   dataset            & deviation \\
                80.6       &   6            &   \\
                74.7       &   2        & 0.3 \\
                72.1       &   3      & 0.5 \\
                55.3       &   4      & 1.7 \\
                73.1       &   5       & 0.4 \\
        };

    \end{axis}
\end{tikzpicture} 
\end{minipage}
\centering
\begin{tikzpicture}
\begin{axis}[
    xbar,
    hide axis,
    transpose legend,
    legend columns = 2,
    xmin = 1, xmax = 2,
    ymin = 1, ymax = 2,
    legend image code/.code={
            \draw [#1] (-0.1cm,-0.1cm) rectangle (0.25cm,0.1cm); },
    legend style={cells={align=left}, legend cell align = {left}}
    ]
    \addlegendimage{blue!80!black,fill=blue!30!white}
    \addlegendentry{Our method};
    \addlegendimage{red!80!black,fill=red!30!white}
    \addlegendentry{Random forest};
    \addlegendimage{brown!80!black,fill=brown!30!white}
    \addlegendentry{XGBoost};
    \addlegendimage{teal!80!black,fill=teal!30!white}
    \addlegendentry{Formula size};
    \addlegendimage{dashed, purple!80!black,fill=purple!10!white}
    \addlegendentry{Literature};
\end{axis}
\end{tikzpicture}
\caption{The average test accuracies obtained with our method, random forests and XGBoost for all datasets. For all but the Colon dataset, also standard deviations are reported. For the UCI datasets we also include a comparison to the formula size method. When available, we have also included accuracies reported in literature, though these are not directly comparable due to different technical particularities in the experiments.}
\label{fig:resultsum}
\end{figure}

We compared our method empirically to random forests and XGBoost. These two comparison methods are very commonly used and state-of-the-art for tabular data. We tested all three methods on $12$ tabular datasets. The datasets can be categorized as follows.
\begin{enumerate}
    \item $7$ binary classification tasks from the UCI machine learning repository. Five of these were selected arbitrarily, while two further ones (BreastCancer and GermanCredit) were randomly chosen among the ones used in \cite{jelia23}. One of the 7 datasets (StudentDropout) was originally a ternary classification task; we converted it into a binary one. 
    \item $4$ tabular-data benchmarks out of $7$ binary classification benchmarks that were presented in the paper \cite{grinsztajn2022why}.
    These datasets were also originally from the UCI machine learning repository.
    \item a high-dimensional binary classification task containing biological data from an open-source feature selection repository presented in \cite{li2018feature}.
\end{enumerate}
All rows with missing values were removed from these datasets.\footnote{In the case of the Hepatitis dataset we removed two columns that had several missing values.} Most of the datasets contained both categorical and numerical features. While our method includes a rough median-based Booleanization of the data as discussed in Section \ref{sec:overfitting}, random forests and XGBoost used the original non-Booleanized data.


For all but one of the datasets we use ten-fold cross-validation for all methods. That is, we split the data into ten equal parts and for each such 10 percent part we input the remaining 90 percent into the method being tested, obtain a classifier (using the chosen 90 percent of the data) and record its accuracy on the 10 percent part. We report the averages and standard deviations of the accuracies over the ten 90/10 splits. 

For the high-dimensional dataset Colon with less than 100 data points, we use leave-one-out cross-validation. That is, for each data point, we set the point aside and input the remaining data into the method. We test whether the obtained classifier classifies the omitted point correctly or not. We report the average accuracy over all data points. This is a well-known standard practice for small datasets.

As mentioned in Section \ref{sec:overfitting}, our method utilizes readily available feature selection methods. We used Scikit-learn's \texttt{SelectKBest}, which returns $k$ features that have the highest scores based on a given univariate statistical test. \texttt{SelectKBest} supports three methods for calculating the scores for the features: F-test, mutual information, and $\chi^2$-test. When generating formulas, we tested all three of these methods and selected the best in terms of validation accuracy. We emphasize that feature selection was performed after Booleanizing the data.

If the empirical ideal classifier is allowed to use too many features, it will most likely overfit on its training data. As already discussed, we used nested cross-validation to determine how many features the empirical ideal classifier can use without overfitting. For nested cross-validation we used a $70/30$-split. In all of the experiments we allowed the empirical ideal classifier to use at most ten attributes. This is a hyperparameter one could optimize for each dataset separately. For simplicity we used the same limit for all datasets. 
In all of our experiments, the validation accuracy either declined or stagnated well before ten features, indicating that a larger number is not needed. On the other hand, for interpretability, ten features is perhaps even a bit excessive.

For random forests we used Scikit-learn's implementation. For both random forests and XGBoost, nested cross-validation was used for tuning the hyperparameters. As for our method, we used a \(70/30\)-split. For hyperparameter optimization, we used Optuna \cite{optuna_paper} with the number of trials being \(100\). The hyperparameter spaces used for random forests and XGBoost were the same as in \cite{grinsztajn2022why}. For the readers convenience, we also report the used hyperparameter spaces in Appendix \ref{appendix:hyperparameter_spaces}.


We also ran the method of \cite{jelia23} for some datasets. As in the Introduction, we refer to their approach as the \emph{formula-size method.} This method is computationally resource consuming, so we only ran it for the seven UCI repository datasets. The high-dimensional biological data and benchmark data would be too difficult for the method.


For the formula-size method we used the openly available implementation from \url{https://github.com/asptools/benchmarks}. The runs were conducted on a Linux cluster featuring Intel Xeon 2.40 GHz CPUs with 8 CPU cores per run, employing a timeout of 72 hours per instance and a memory limit of 64 GB.



\subsection{Results}\label{sec:results}

Figure \ref{fig:resultsum} contains the average test accuracies and standard deviations that were obtained by using our new method, random forests and XGBoost using ten-fold cross-validation. The left side of Figure \ref{fig:resultsum} also contains results obtained with the formula-size method of \cite{jelia23}. We also report accuracies that we found in the literature for these datasets. For the left side of Figure \ref{fig:resultsum}, we include the accuracies that were reported for these datasets in the UCI machine learning repository. For the right side of Figure \ref{fig:resultsum}, for the case of Colon, we used the result reported in \cite{Yang2021LocallySN}, while for remaining datasets we reported the accuries given in the full-version of \cite{grinsztajn2022why}.

We emphasize that the accuracies found from the literature are not directly comparable to the accuracies that we obtained with our method. For example, in some cases a single $70/30$-split was used instead of our ten-fold cross-validation. The conventions concerning the handling of missing values could also differ from ours. Furthermore, the accuracies in \cite{grinsztajn2022why} are given to the nearest percentage point.

From Figure \ref{fig:resultsum} we see that our method produces classifiers that mostly have accuracies similar to the ones obtained by  the alternative methods. In particular, on medical data such as BreastCancer and Hepatitis, our method compares well to the reference methods.
Perhaps the biggest gaps are found in the cases of CoverType and Electricity, two of the benchmark datasets. However, we emphasize that in \cite{grinsztajn2022why}, it was mentioned that all of the benchmark datasets were specifically selected to be ``not too easy''. Hence, one should not perhaps expect that ``simple'' classifiers with focus on interpretability will perform too well with these datasets.

We list the average runtime of our method on a single split of the data in the Appendix \ref{appendix:runtimes}. For six of the datasets, the average runtime was less than 30 seconds. Of the remaining datasets, four had  runtimes of less than 12 minutes. Even for the two largest datasets, Covertype and RoadSafety, the runtime is still less than an hour and, notably, systematically less than that of the hyperparameter optimizations of random forests and XGBoost. The experiments on our method, random forests and XBGoost were run on a laptop.

We turn our attention to our main goal of interpretability. 
In Table \ref{tab:featurenumbers} we report the average of the \emph{smallest number of features} sufficient to reach the reported accuracy. Table \ref{tab:featurenumbers} also gives the number of rows and attributes in the Booleanized datasets as well as the original datasets. 
We see from Table \ref{tab:featurenumbers} that for eight out of 12 datasets, the best accuracy was reached already with on average less than four features, leading to very short formulas. Such short DNF-formulas are globally interpretable as follows. To interpret a short DNF-formula, one looks at each of the few conjunctions in the formula individually to determine their meaning. This is easy as each conjunction has very few attributes. The meaning of the entire formula is then given by the fact that it accepts only the cases given by the conjunctions and rejects everything else.





Next we discuss the results obtained using the formula-size method of \cite{jelia23}. By the size of a propositional formula they mean the number of proposition symbols, conjunctions, disjunctions and negations in the formula. The results are given in Figure \ref{fig:resultsum}.

We see that for most datasets on the left side in Figure \ref{fig:resultsum}, the accuracies obtained by the formula-size method are similar to our method. A meaningful difference can be seen in the HeartDisease dataset, where both the accuracy of 73.6 percent and standard deviation of 12.2 percent are significantly behind other methods. The formulas obtained are short, with formula size mostly less than 10.

The downside of the formula-size method is computational resources. As described in the previous subsection, we used a Linux cluster featuring Intel Xeon 2.40 GHz CPUs with 8 CPU cores per run, employing a timeout of 72 hours per instance and a memory limit of 64 GB. For all but one dataset we ran the formula-size method on, the computation took more than 24 hours per split of the data with this high efficiency setup. Even for the easiest one, BreastCancer, the computation time was in the order of 12 hours per split. Furthermore, for the benchmark datasets and the high-dimensional biological dataset Colon, the formula-size method is unfeasible already for very small formula sizes and thus we do not report any results for those datasets.



In contrast to the formula-size method, for all but the two largest datasets, the runs of our method took less than 12 minutes per split, and in six cases only seconds. More detailed runtimes are reported in Appendix \ref{appendix:runtimes}. We conclude that our method obtains similar or better results than the formula-size method with a fraction of the computational resources on datasets where both methods are feasible. Our method is also usable on larger datasets, where the formula-size method is is not.

\begin{table}[tb]
    \begin{center}
    {\caption{The average number of features used by the final formulas for each dataset. We also report the total numbers of features in the Booleanized and original datasets for comparison as well as the numbers of data points.}
    \label{tab:featurenumbers}}
    \vspace{10pt}
    \begin{tabular}{|l|c|c|c|c|}
        \hline 
        Data set & \makecell{Selected \\ features} & \makecell{Boolean \\ features} & \makecell{Original \\ features} & \makecell{Data \\ points} \\
        \hline
        BankMarketing & 1.3 & 51 & 51 & 4521 \\
        \hline
        BreastCancer & 3.8 & 9 & 9 & 683 \\
        \hline
        CongressionalVoting & 1.0 & 48 & 16 & 435 \\
        \hline
        GermanCredit & 3.9 & 61 & 20 & 1000 \\
        \hline
        HeartDisease & 3.9 & 25 & 13 & 303 \\
        \hline
        Hepatitis & 2.7 & 74 & 17 & 155 \\
        \hline
        StudentDropout & 5.8 & 112 & 36 & 4424 \\
        \hline
        Covertype & 2.5 & 54 & 54 & 423680 \\
        \hline
        Electricity & 7.5 & 14 & 8 & 38474 \\
        \hline
        EyeMovement & 5.5 & 26 & 23 & 7608 \\
        \hline
        RoadSafety & 5.4 & 324 & 32 & 111762 \\
        \hline
        Colon & 2.1 & 5997 & 2000 & 63 \\
        \hline
    \end{tabular}
    \end{center}
\end{table}

\subsection{Examples of obtained formulas}\label{ssec:examples}

Here we present selected examples of formulas that were generated by our method.
%
%
Firstly, we refer the reader back to the example formulas concerning the BreastCancer and Colon datasets presented in the Introduction. Both of these formulas were very short and had high accuracies.

Another example formula is from the HeartDisease dataset, where the task is to determine,
whether a patient has a heart disease. We obtain the formula
\[
(p_1 \land p_2) \lor (p_1 \land \neg p_3) \lor (p_2 \land \neg p_3),
\]
where the attributes $p_1$, $p_2$ and $p_3$ relate to a color test of blood vessels, fixed thallium defects and chest pain, respectively. This formula is obtained from two different splits of the data and the respective accuracies are 76.7 percent and 86.7 percent. The average accuracy of our method on this dataset is 81.2 percent. Random forests average at 81.8 percent, XGBoost at 80.8 percent and the formula size method at 73.6 percent.


The Hepatitis dataset concerns the mortality of patients with hepatitis. The formula
\[
p_1 \lor \neg p_2 \lor \neg p_3
\]
is obtained from five of the ten splits of the data. Here $p_1$, $p_2$ and $p_3$ relate to albumin, bilirubin and a histology test respectively. The accuracies of this formula on the five different test datasets range from 61.5 percent to 92.3 percent. The average accuracy of our method on the Hepatitis dataset is 80.7 while random forests get 78.2 percent and XGBoost 79.0 percent. The high variance seems to be due to the data only consisting of 155 data points. This is supported by the fact that all of the tested methods obtain a high standard deviation of accuracy on this dataset.


In the Appendix \ref{appendix:first_formulas}, there are further examples of formulas obtained as outputs of our method. For each dataset, we list the formula obtained for the first 90/10 split of our ten-fold cross-validation.  
While some of these formulas are quite long, they should still be rather fast to use for local explanations  explicating why a particular input was classified in a particular way (cf. the Introduction for details on local explanations). For these experiments, we used the same maximum number 10 of features and the same limit of 1 percentage point to choose an accurate enough formula. These are hyperparameters that one could optimize to obtain a better balance of interpretability and accuracy on any specific dataset. For some datasets, raising the number of maximum features could lead to more accurate, but longer, formulas. Raising the percentage threshold on the other hand would lead to shorter, more interpretable formulas at the cost of some accuracy.


\section{Conclusions}\label{sec:conclusions}

We have introduced a new method to compute immediately interpretable Boolean classifiers for tabular data. While the main point is interpretability, even the accuracy of our formulas is similar to the ones obtained via widely recognized current methods. We have also established a theoretical result for estimating if the obtained formulas actually correspond to ideally accurate ones in relation to the background distribution. 
In the future, we will especially consider more custom-made procedures of discretization in the context of our method, as this time discretization was carried out in a very rough way to demonstrate the effectiveness and potential of our approach. We expect this to significantly improve our method over a notable class of tabular datasets. 

Our approach need not limit to Boolean formulas only, as we can
naturally extend our work to general relational data. We can use,
e.g., description logics and compute concepts $C_1,\dots , C_k$ and
then perform our procedure using $C_1,\dots, C_k$, finding short
Boolean combinations of concepts. This of course differs from the
approach of computing empirical ideal classifiers in the original
description logic, but can nevertheless be fruitful and
interesting. We leave this for future work.







\bibliographystyle{plain}
\bibliography{mybibfile}


\section{Appendix.}

\subsection{Proof of Theorem \ref{thm:sample_size}}\label{appendix:proof_of_the_theorem}

Let $\tau$ be a propositional vocabulary and $q \not\in \tau$ the proposition symbol that we need to explain. Let $\tau^+ := \tau \cup \{q\}$ and fix a probability distribution $\mu : T_{\tau^+} \to [0,1]$. In what follows, for notational simplicity we will assume that $\mu(t \land q) > \mu(t \land \neg q)$, for every $\tau$-type $t$.

Let $M$ be a sample of size $n$, i.e., a $\tau^+$-model of size $n$. For each $\tau^+$-type $t$ we use $\llbracket t \rrbracket_M$ to denote the number of times $t$ is realized in $M$. The main idea of the proof is to show that if 
\[n \geq \frac{2\ln(2^{|\tau| + 1}/\delta)}{\varepsilon^2},\]
then 
\begin{align*}\Pr\bigg[\text{for every $\tau$-type $t$ for which $\mu(t \land q) - \mu(t \land \neg q) \geq \varepsilon$}\\ \text{we have that } \llbracket t \land q \rrbracket_M > \llbracket t \land \neg q \rrbracket_M \bigg] \geq 1 - \delta.
\end{align*} 
Indeed, the above clearly implies that also with probability at least $1 - \delta$ we have that the empirical ideal classifier with respect to $M$ agrees with the ideal classifier with respect to all $\tau$-types $t$ with $\mu(t \land q) - \mu(t \land \neg q) \geq \varepsilon$.

Consider first a fixed $\tau$-type $t$ such that $\mu(t \land q) - \mu(t \land \neg q) \geq \varepsilon$. Now, observe that if
\[n(\mu(t \land q) - \theta) < \llbracket t \land q \rrbracket_M \text{ and } \llbracket t \land \neg q \rrbracket_M < n(\mu(t \land \neg q) + \theta),\]
then
\[\llbracket t \land q \rrbracket_M > \llbracket t \land \neg q \rrbracket_M\]
holds provided that
\[\mu(t \land q) - \mu(t \land \neg q) \geq 2\theta.\]
Setting $\theta := \varepsilon/2$, we have by assumption that
\begin{align*}
    & \Pr[\llbracket t \land q \rrbracket_M \geq \llbracket t \land \neg q \rrbracket_M] \\
    & \geq \Pr[n(\mu(t \land q) - \theta) < \llbracket t \land q \rrbracket_M \text{ and } \llbracket t \land \neg q \rrbracket_M < n(\mu(t \land \neg q) + \theta)] \\
    & \geq 1 - (\Pr[n(\mu(t \land q) - \theta) \geq \llbracket t \land q \rrbracket_M] + \Pr[\llbracket t \land \neg q \rrbracket_M \geq n(\mu(t \land \neg q) + \theta)])
\end{align*}
For the second inequality we used the union bound. To get a lower bound of $1 - \delta$ on the latter probability, it suffices to show that
\[\Pr[n(\mu(t \land q) - \theta) \geq \llbracket t \land q \rrbracket_M] \leq \frac{\delta}{2}\]
and
\[\Pr[\llbracket t \land \neg q \rrbracket_M \geq n(\mu(t \land \neg q) + \theta)] \leq \frac{\delta}{2}\]
Hoeffding's inequalities \cite{Hoeffding} imply that these bounds hold provided that 
\[n \geq \frac{\ln(2/\delta)}{2\theta^2} = \frac{2\ln(2/\delta)}{\varepsilon^2}.\]
Thus, if we take sample $M$ of at least size
\[\frac{2\ln(2/\delta)}{\varepsilon^2},\]
then with probability at least $1 - \delta$ we have that $\llbracket t \land q \rrbracket_M > \llbracket t \land \neg q \rrbracket_M$.

So far we focused on a single $\tau$-type $t$. Using union bound again we see that 
\begin{align*}
     \Pr\bigg[\text{for every $\tau$-type $t$ for which $\mu(t \land q) - \mu(t \land \neg q) \geq \varepsilon$} & \\
       \text{we have that }  \llbracket t  \land q \rrbracket_M > \llbracket t \land \neg q \rrbracket_M \bigg] \\
     \geq 1 - \sum \Pr[\llbracket t & \land q \rrbracket_M \leq \llbracket t \land q \rrbracket_M],
\end{align*}
where the sum is performed with respect to $\tau$-types $t$ for which $\mu(t \land q) - \mu(t \land \neg q) \geq \varepsilon$. Setting again $\theta := \varepsilon/2$ it follows from our previous calculations that for every $\tau \in T_\tau$ for which $\mu(t \land q) - \mu(t \land \neg q) \geq \varepsilon$ we have that
\[\Pr[\llbracket t \land q \rrbracket_M \leq \llbracket t \land q \rrbracket_M] \leq \delta/2^{|\tau|},\]
provided that
\[n \geq \frac{\ln(2^{|\tau| + 1}/\delta)}{2\theta^2} = \frac{2\ln(2^{|\tau| + 1}/\delta)}{\varepsilon^2}.\]
In particular, we will then have that 
\begin{align*}\Pr\bigg[\text{for every $\tau$-type $t$ for which $\mu(t \land q) - \mu(t \land \neg q) \geq \varepsilon$}\\
\text{ we have that } \llbracket t \land q \rrbracket_M > \llbracket t \land \neg q \rrbracket_M \bigg] \geq 1 - \delta,
\end{align*}
which is what we wanted to show.

\subsection{Hyperparameter spaces for Random Forest and XGBoost}\label{appendix:hyperparameter_spaces}

Hyperparameters not listed here were kept at their default values.

\subsubsection{Random Forest}

\begin{itemize}
    \item \texttt{max\_depth}: None, 2, 3, 4
    \item \texttt{n\_estimators}: Integer sampled from \([9.5,3000.5]\) using log domain
    \item \texttt{criterion}: gini, entropy
    \item \texttt{max\_features}: sqrt, log2, None, 0.1, 0.2, 0.3, 0.4, 0.5, 0.6, 0.7, 0.8, 0.9
    \item \texttt{min\_samples\_split}: 2, 3
    \item \texttt{min\_samples\_leaf}: Integer sampled from \([1.5,50.5]\)
    \item \texttt{bootstrap}: True, False
    \item \texttt{min\_impurity\_decrease}: 0.0, 0.01, 0.02, 0.05
\end{itemize}

\subsubsection{XGBoost}

\begin{itemize}
    \item \texttt{max\_depth}: Integer sampled from [1,11]
    \item \texttt{n\_estimators}: Integer sampled from [100,5900] with step-size being 200.
    \item \texttt{min\_child\_weight}: Float sampled from [1.0, 100.0] using log domain
    \item \texttt{subsample}: Float sampled from [0.5, 1.0]
    \item \texttt{learning\_rate}: Float sampled from [1e-5, 0.7] using log domain
    \item \texttt{colsample\_bylevel}: Float sampled from [0.5, 1.0]
    \item \texttt{gamma}: Float sampled from [1e-8, 7.0] using log domain
    \item \texttt{reg\_lambda}: Float sampled from [1.0, 4.0] using log domain
    \item \texttt{reg\_alpha}: Float sampled from [1e-8, 100.0] using log domain
\end{itemize}

\subsection{Runtimes}\label{appendix:runtimes}

We report the average runtime of our method over the different splits of the ten-fold (or in the case of Colon leave-one-out) cross-validation. The experiments were run on a standard laptop.

    \begin{center}
    \vspace{10pt}
    \begin{tabular}{|l|c|}
        \hline 
        Data set & \makecell{Runtime of our method} \\
        \hline
        BankMarketing & 1 minute 30 seconds  \\
        \hline
        BreastCancer & 13 seconds  \\
        \hline
        CongressionalVoting & 10 seconds  \\
        \hline
        GermanCredit & 26 seconds  \\
        \hline
        HeartDisease & 11 seconds \\
        \hline
        Hepatitis & 2 seconds  \\
        \hline
        StudentDropout & 1 minute 26 seconds  \\
        \hline
        Covertype & 45 minutes 10 seconds \\
        \hline
        Electricity & 11 minutes 54 seconds  \\
        \hline
        EyeMovement & 3 minutes 0 seconds \\
        \hline
        RoadSafety & 28 minutes 37 seconds \\
        \hline
        Colon & 19 seconds \\
        \hline
    \end{tabular}
    \end{center}

\subsection{Example formula for each dataset}\label{appendix:first_formulas}

Here we report for each dataset the formula given by the first of the ten splits in our ten-fold cross-validation. 

\subsubsection{BankMarketing}

\[
poutcome\_success
\]

\subsubsection{BreastCancer}

\scalebox{0.7}{
\begin{minipage}{\linewidth}
\begin{align*}
&(\neg bare\_nuclei\_above\_median \land \neg bland\_chromatin\_above\_median) \\
&\lor (\neg bland\_chromatin\_above\_median \land \neg single\_epithelial\_cell\_size\_above\_median) \\
&\lor (\neg bland\_chromatin\_above\_median \land \neg uniformity\_of\_cell\_size\_above\_median) \\
&\lor (bare\_nuclei\_above\_median \land \neg single\_epithelial\_cell\_size\_above\_median \land \neg uniformity\_of\_cell\_size\_above\_median) \\
&\lor (single\_epithelial\_cell\_size\_above\_median \land \neg bare\_nuclei\_above\_median \land \neg uniformity\_of\_cell\_size\_above\_median) \\
&\lor (uniformity\_of\_cell\_size\_above\_median \land \neg bare\_nuclei\_above\_median \land \neg single\_epithelial\_cell\_size\_above\_median)
\end{align*}
\end{minipage}
}

\subsubsection{CongressionalVoting}

\[
physician\_fee\_freeze\_y
\]

\subsubsection{GermanCredit}

\begin{align*}
    &(A1\_A14 \land \neg A1\_A11 \land \neg A3\_A30) \\
    &\lor (A1\_A14 \land \neg A1\_A11 \land \neg A3\_A31) \\
    &\lor (A6\_A65 \land \neg A1\_A11 \land \neg A3\_A30) \\
    &\lor (\neg A1\_A14 \land \neg A3\_A30 \land \neg A3\_A31) \\
    &\lor (A1\_A11 \land A6\_A65 \land \neg A1\_A14 \land \neg A3\_A31)
\end{align*}

\subsubsection{HeartDisease}

\[
(caa\_0 \land thall\_2) \lor (caa\_0 \land \neg oldpeak\_above\_median) \lor (thall\_2 \land \neg oldpeak\_above\_median)
\]

\subsubsection{Hepatitis}

\[
albumin\_above\_median \lor \neg bilirubin\_above\_median \lor \neg histology
\]

\subsubsection{StudentDropout}

\scalebox{0.7}{
\begin{minipage}{\linewidth}
\begin{align*}
&(\neg Curricular\_units\_1st\_sem\_grade\_above\_median \land \neg Tuition\_fees\_up\_to\_date) \\
&\lor (\neg Curricular\_units\_2nd\_sem\_approved\_above\_median \land \neg Tuition\_fees\_up\_to\_date) \\
&\lor (\neg Curricular\_units\_1st\_sem\_approved\_above\_median \land \neg Curricular\_units\_1st\_sem\_grade\_above\_median \\
&\ \ \ \ \ \land \neg Curricular\_units\_2nd\_sem\_approved\_above\_median \land \neg Curricular\_units\_2nd\_sem\_grade\_above\_median)
\end{align*}
\end{minipage}
}

\subsubsection{Covertype}

\[(Elevation\_above\_median \land \neg Soil\_Type12) \lor (Soil\_Type22 \land \neg Soil\_Type12)\]

\subsubsection{Electricity}

\scalebox{0.6}{
\begin{minipage}{\linewidth}
\begin{align*}{
&(nswprice\_above\_median \land transfer\_above\_median \land vicprice\_above\_median) \\
&\lor (nswprice\_above\_median \land vicprice\_above\_median \land \neg nswdemand\_above\_median) \\
&\lor (day\_0 \land nswdemand\_above\_median \land period\_above\_median \land transfer\_above\_median \land vicdemand\_above\_median) \\
&\lor (nswdemand\_above\_median \land nswprice\_above\_median \land \neg day\_0 \land \neg period\_above\_median) \\
&\lor (nswdemand\_above\_median \land nswprice\_above\_median \land \neg vicdemand\_above\_median \land \neg vicprice\_above\_median) \\
&\lor (nswprice\_above\_median \land period\_above\_median \land \neg day\_0 \land \neg vicdemand\_above\_median) \\
&\lor (nswprice\_above\_median \land vicdemand\_above\_median \land \neg day\_0 \land \neg transfer\_above\_median) \\
&\lor (day\_0 \land nswdemand\_above\_median \land period\_above\_median \land vicprice\_above\_median \land \neg transfer\_above\_median) \\
&\lor (nswdemand\_above\_median \land period\_above\_median \land transfer\_above\_median \land vicprice\_above\_median \land \neg day\_0) \\
&\lor (day\_0 \land period\_above\_median \land transfer\_above\_median \land \neg nswprice\_above\_median \land \neg vicprice\_above\_median) \\
&\lor (nswdemand\_above\_median \land vicdemand\_above\_median \land vicprice\_above\_median \land \neg day\_0 \land \neg period\_above\_median) \\
&\lor (day\_0 \land nswprice\_above\_median \land \neg period\_above\_median \land \neg vicdemand\_above\_median \land \neg vicprice\_above\_median) \\
&\lor (day\_0 \land nswprice\_above\_median \land \neg transfer\_above\_median \land \neg vicdemand\_above\_median \land \neg vicprice\_above\_median) \\
&\lor (day\_0 \land vicprice\_above\_median \land \neg nswdemand\_above\_median \land \neg period\_above\_median \land \neg vicdemand\_above\_median) \\
&\lor (nswdemand\_above\_median \land transfer\_above\_median \land \neg period\_above\_median \land \neg vicdemand\_above\_median \land \neg vicprice\_above\_median) \\
&\lor (vicdemand\_above\_median \land \neg day\_0 \land \neg nswdemand\_above\_median \land \neg period\_above\_median \land \neg transfer\_above\_median) \\
&\lor (day\_0 \land vicdemand\_above\_median \land \neg nswdemand\_above\_median \land \neg nswprice\_above\_median \land \neg transfer\_above\_median \land \neg vicprice\_above\_median)
}\end{align*}
\end{minipage}
}

\subsubsection{EyeMovement}

\scalebox{0.7}{
\begin{minipage}{\linewidth}
\begin{align*}
    &(nextWordRegress\_0 \land \neg nextWordRegress\_1 \land \neg wordNo\_above\_median) \\
    &\lor (nextWordRegress\_0 \land prevFixPos\_above\_median \land \neg nextWordRegress\_1 \land \neg regressDur\_above\_median) \\
    &\lor (nextWordRegress\_1 \land \neg nextWordRegress\_0 \land \neg prevFixPos\_above\_median \land \neg regressDur\_above\_median \\
    &\ \ \ \ \ \ \ \land \neg wordNo\_above\_median)
\end{align*}
\end{minipage}
}

\subsubsection{RoadSafety}\label{app:roadsafetyformula}

\scalebox{0.63}{
\begin{minipage}{\linewidth}
\begin{align*}
    &(Casualty\_Type\_9 \land Sex\_of\_Casualty\_1 \land Vehicle\_Type\_9 \land \neg Sex\_of\_Casualty\_0) \\
    &\lor (Propulsion\_Code\_1 \land Sex\_of\_Casualty\_1 \land Vehicle\_Type\_9 \land \neg Sex\_of\_Casualty\_0) \\
    &\lor (Sex\_of\_Casualty\_1 \land Vehicle\_Type\_9 \land \neg Engine\_Capacity\_CC\_above\_median \land \neg Sex\_of\_Casualty\_0) \\
    &\lor (Propulsion\_Code\_1 \land Sex\_of\_Casualty\_0 \land Vehicle\_Type\_9 \land \neg Casualty\_Type\_9 \land \neg Sex\_of\_Casualty\_1) \\
    &\lor (Propulsion\_Code\_1 \land Sex\_of\_Casualty\_1 \land \neg Casualty\_Type\_9 \land \neg Engine\_Capacity\_CC\_above\_median \land \neg Sex\_of\_Casualty\_0) \\
    &\lor (Sex\_of\_Casualty\_0 \land Vehicle\_Type\_9 \land \neg Casualty\_Type\_9 \land \neg Engine\_Capacity\_CC\_above\_median \land \neg Sex\_of\_Casualty\_1)
\end{align*}
\end{minipage}
}

\subsubsection{Colon}

For the high-dimensional dataset Colon we used leave-one-out cross-validation; we report the formula obtained from leaving the first datapoint out.

\[v765\_two\]

\end{document}